\documentclass[12pt, a4paper]{article}
\usepackage{tikz} 

\usepackage{graphicx, subcaption, graphicx}%
\usepackage{multirow}%
\usepackage{amsmath,amssymb,amsfonts, amsthm}%
\usepackage{mathrsfs}%
\usepackage[title]{appendix}%
\usepackage{xcolor}%
\usepackage[labelfont=it,textfont=it]{caption}
\usepackage{textcomp}%
\usepackage{booktabs}%
\usepackage{algorithm}%
\usepackage{algorithmicx, algpseudocode}%
\usepackage{listings, float}%
\usepackage{multirow} 
\usepackage{tikz} 
\usepackage{multicol} 
\usepackage{geometry} 
\newtheorem{thm}{Theorem}

\newtheorem{defn}{Definition}%
\usepackage{hyperref, enumitem}
\usepackage{titlesec}
\usetikzlibrary{positioning, arrows.meta, shapes.geometric}

\begin{document}
	
	\title{Adaptive Error-Bounded Hierarchical Matrices for Efficient Neural Network Compression}
	\author{John Mango, Ronald Katende}
	\date{}

\maketitle
	
	\begin{abstract}This paper introduces a dynamic, error-bounded hierarchical matrix (H-matrix) compression method tailored for Physics-Informed Neural Networks (PINNs). The proposed approach reduces the computational complexity and memory demands of large-scale physics-based models while preserving the essential properties of the Neural Tangent Kernel (NTK). By adaptively refining hierarchical matrix approximations based on local error estimates, our method ensures efficient training and robust model performance. Empirical results demonstrate that this technique outperforms traditional compression methods, such as Singular Value Decomposition (SVD), pruning, and quantization, by maintaining high accuracy and improving generalization capabilities. Additionally, the dynamic H-matrix method enhances inference speed, making it suitable for real-time applications. This approach offers a scalable and efficient solution for deploying PINNs in complex scientific and engineering domains, bridging the gap between computational feasibility and real-world applicability.
	
	{\bf{Keywords:}} Physics-Informed Neural Networks, Hierarchical Matrices, Neural Tangent Kernel, Model Compression, Error-bounded Approximation
	
\end{abstract}

	\section{Introduction}
	Neural networks, particularly Physics-Informed Neural Networks (PINNs), have rapidly gained traction across various scientific and engineering disciplines due to their ability to integrate physical laws directly into the learning process \cite{hm1, hm2, hm3, hm4, hm5}. However, as the complexity of these networks increases, so too do the computational demands required for their training and deployment \cite{hm6, hm19, hm10}. A primary contributor to this challenge is the large, dense matrices involved in the neural network architecture, which lead to high computational complexity and significant memory usage \cite{hm16, hm16, hm18, hm20, hm22}.
	
	To mitigate these challenges, hierarchical matrices (H-matrices) provide a promising solution by approximating large matrices with structured, hierarchical representations. This approach reduces computational costs while preserving accuracy \cite{hm23, hm24, hm25, hm26, hm27}. H-matrices leverage low-rank properties within matrix sub-blocks to enable more efficient storage and computation, making them particularly well-suited for large-scale problems \cite{hm28, hm29, hm30}. Despite the potential of hierarchical matrices, their application within neural networks—particularly PINNs—remains underexplored \cite{hm30, hm26, hm28, hm29, hm27, hm25, hm22, hm33, hm24}.
	
	This paper introduces a novel adaptive hierarchical matrix construction method with integrated error estimation. Our approach dynamically adjusts the hierarchical block structure based on local error estimates, ensuring robust and efficient training for PINNs. Unlike traditional methods, our adaptive refinement algorithm enhances both stability and computational efficiency, while preserving critical properties of the Neural Tangent Kernel (NTK), which plays a key role in understanding neural network training dynamics.
	
	Through theoretical analyses and empirical validation, we demonstrate that our method not only maintains convergence and model fidelity but also ensures resilience to perturbations. This makes our approach a valuable tool for enhancing neural network performance across a wide range of applications. Specifically, we show that our adaptive hierarchical matrix construction method reduces computational complexity and memory requirements while maintaining high predictive performance.
	
	The structure of this manuscript is as follows: First, we provide an overview of key preliminaries, including the integration of hierarchical matrices within PINNs. We then introduce our adaptive hierarchical matrix construction method, followed by a detailed description of the adaptive refinement algorithm. Subsequently, we discuss how hierarchical matrix-based regularization can further improve PINNs and provide numerical examples that illustrate the effectiveness of our approach. Finally, we conclude with a discussion of the broader implications of our findings and outline potential directions for future research.
	
	By addressing the computational challenges inherent in large-scale neural networks and offering a robust framework for their efficient and accurate deployment, this work represents a significant advancement in the field. Our adaptive hierarchical matrix construction method with error estimation provides a powerful tool for enhancing the performance of PINNs, thus expanding their applicability to complex, real-world scientific and engineering problems.

	\section{Preliminaries}
	The Neural Tangent Kernel (NTK) plays a fundamental role in understanding the training dynamics of neural networks \cite{hm25, hm10, hm20}. When hierarchical matrices are used to approximate weight matrices in a Physics-Informed Neural Network (PINN), it becomes essential to evaluate how this approximation impacts NTK properties. The NTK for a network parameterized by \(\theta\) is expressed as
	\[
	\Theta(x, x') = \nabla_\theta f(x, \theta) \nabla_\theta f(x', \theta)^\top,
	\]
	where \( f(x, \theta) \) represents the network's output for input \(x\). When hierarchical matrices are integrated into the PINN framework, the NTK transforms to
	\[
	\Theta_h(x, x') = \nabla_{\theta_h} f(x, \theta_h) \nabla_{\theta_h} f(x', \theta_h)^\top,
	\]
	requiring a detailed analysis of how hierarchical matrices affect gradient structure and low-rank approximations. Ensuring that the convergence properties are preserved is critical for maintaining the training efficiency of PINNs. 
	
	The standard gradient descent update rule for a neural network is:
	\[
	\theta^{(t+1)} = \theta^{(t)} - \eta \nabla_\theta L(\theta),
	\]
	where \(L(\theta)\) denotes the loss function, and \(\eta\) is the learning rate. When hierarchical matrices are applied, the update rule becomes:
	\[
	\theta_h^{(t+1)} = \theta_h^{(t)} - \eta \nabla_{\theta_h} L(\theta_h),
	\]
	which ensures that stability and computational efficiency are maintained during training. Moreover, a thorough evaluation of predictive performance and error propagation is necessary when using hierarchical matrices to approximate neural network weights. This analysis guarantees that the NTK properties, convergence, and predictive accuracy are preserved, while computational efficiency is significantly improved.
	
	\subsection{Integration of Hierarchical Matrices with PINNs}
	The NTK is pivotal for analyzing the training dynamics of neural networks \cite{hm25, hm10, hm20}. When hierarchical matrices are used to approximate the weight matrices in a PINN, it is crucial to assess their impact on the NTK properties. For a network parameterized by \(\theta\), the NTK is given by:
	\[
	\Theta(x, x') = \nabla_\theta f(x, \theta) \nabla_\theta f(x', \theta)^\top,
	\]
	where \( f(x, \theta) \) represents the network's output.
	
	\begin{figure}[ht!]
		\centering
		\begin{tikzpicture}[
			node distance=1.0cm and 2cm,
			every node/.style={draw, rounded corners, align=center},
			input/.style={trapezium, trapezium left angle=70, trapezium right angle=110},
			process/.style={rectangle, minimum width=2.5cm, minimum height=1cm},
			decision/.style={diamond, aspect=2, minimum width=2cm, minimum height=1cm},
			data/.style={parallelogram, minimum width=2.5cm, minimum height=1cm},
			io/.style={rectangle, draw, fill=blue!10},
			arrow/.style={-Stealth, thick}
			]
			\node (start) [io] {Define PDE problem};
			\node (collect) [input, below=of start] {Collect Data \\ $(x_i, u_i)$};
			\node (pinn) [process, below=of collect] {Construct PINN};
			\node (loss) [process, below=of pinn] {Define Loss Function};
			\node (Hmatrix) [process, right=of pinn] {Construct \\ H-matrices};
			\node (integration) [process, below=of loss] {Integrate H-matrix \\ with PINN};
			\node (optimize) [process, below=of integration] {Optimize Loss Function};
			\node (solution) [io, below=of optimize] {Solution};
			
			\draw [arrow] (start) -- (collect);
			\draw [arrow] (collect) -- (pinn);
			\draw [arrow] (pinn) -- (loss);
			\draw [arrow] (pinn) -- (Hmatrix);
			\draw [arrow] (loss) -- (integration);
			\draw [arrow] (Hmatrix) -- (integration);
			\draw [arrow] (integration) -- (optimize);
			\draw [arrow] (optimize) -- (solution);
		\end{tikzpicture}
		\caption{Flowchart of how to incorporate H-matrices into PINNs}
		\label{flow}
	\end{figure}
	
	When hierarchical matrices are integrated, the NTK transforms to:
	\[
	\Theta_h(x, x') = \nabla_{\theta_h} f(x, \theta_h) \nabla_{\theta_h} f(x', \theta_h)^\top,
	\]
	necessitating an analysis of the gradient structure and the effects of low-rank approximation. The weight matrix \(W\) is decomposed into a hierarchical approximation \(H(W)\) and an error matrix \(E\), yielding \( W = H(W) + E \). Consequently, the NTK becomes:
	\[
	\Theta_h(x, x') \approx \Theta(x, x') + \mathcal{O}(\|E\|),
	\]
	where \(\|E\|\) is small. The gradient descent update rule similarly adapts to:
	\[
	\theta_h^{(t+1)} = \theta_h^{(t)} - \eta \nabla_{\theta_h} L(\theta_h),
	\]
	ensuring that NTK properties are preserved and that training remains stable and efficient within the PINN framework.

	\section{Adaptive Hierarchical Matrix Construction with Error Estimation}
	This section introduces an adaptive approach, illustrated in Figure \ref{flow}, for constructing hierarchical matrices (H-matrices) with integrated error estimation. This method enhances robustness, particularly for ill-conditioned matrices.
	
	\begin{defn}[Hierarchical Matrix]
		A hierarchical matrix (H-matrix) is a structured matrix that allows for efficient storage and computation by approximating large matrices through a hierarchy of low-rank matrix sub-blocks. This reduces memory usage and computational costs while maintaining accuracy.
	\end{defn}
	
	Given a matrix \( A \in \mathbb{R}^{n \times n} \), an H-matrix \( H \) approximates \( A \) by decomposing each block \( A_{ij} \) into low-rank matrices \( \tilde{A}_{ij} = U_{ij} V_{ij}^\top \), with local error \( \epsilon_{ij} = \|A_{ij} - \tilde{A}_{ij}\|_F \leq \tau \). The global Frobenius norm error is bounded as:
	\[
	\|A - H\|_F \leq \tau \sqrt{n_r},
	\]
	where \( n_r \) is the number of refined blocks. For a perturbed matrix \( A' = A + \Delta A \), with \( \|\Delta A\|_F \leq \delta \), the error in the H-matrix approximation \( H' = H + \Delta H \) satisfies:
	\[
	\|A' - H'\|_F \leq \tau + \delta.
	\]
	The condition number of the perturbed matrix \( H' \) is bounded by \( \kappa(H') \leq \kappa(H) + \mathcal{O}(\epsilon) \), where \( \epsilon \) controls the perturbation effect. This adaptive method ensures stable error bounds and condition numbers, making it highly effective for ill-conditioned scenarios.
	
	\begin{thm}[Comprehensive Convergence of Adaptive H-Matrix PINNs]
		Consider an ill-conditioned matrix \( A \) approximated by an adaptive H-matrix \( H \) within a Physics-Informed Neural Network (PINN). Let \( \tau \) be the local error threshold, \( \epsilon \) the perturbation error, \( \delta \) the adversarial perturbation bound, and \( T \) the number of adaptive refinement steps. The following results hold:
		
		\begin{enumerate}[label=(\alph*)]
			\item Approximation Error Bound:
			\[
			\|A - H\|_F \leq \sum_{i} \|A_i - H_i\|_F,
			\]
			where \( A_i \) and \( H_i \) are the local blocks of \( A \) and \( H \), respectively.
			
			\item Stability Under Perturbations:  
			For a perturbed matrix \( A' = A + \Delta A \), with \( \|\Delta A\|_F \leq \delta \), the H-matrix \( H' \) approximates \( A' \) with error:
			\[
			\|A' - H'\|_F \leq \tau + \delta.
			\]
			
			\item Condition Number Bound: 
			\[
			\kappa(H) \leq \kappa(A) \left(1 + \frac{\|A - H\|_2}{\sigma_{\min}(A) - \|A - H\|_2}\right),
			\]
			where \( \kappa(H) \) is the condition number of \( H \) and \( \sigma_{\min}(A) \) is the smallest singular value of \( A \).
			
			\item Robustness Under Perturbations in PINNs:  
			For a perturbed H-matrix \( H' \) in PINNs:
			\[
			\kappa(H') \leq \kappa(H) + \mathcal{O}(\epsilon).
			\]
			
			\item Stability Under Adversarial Perturbations:  
			Under adversarial perturbations:
			\[
			\kappa(H') \leq \kappa(H) \left(1 + \frac{\delta}{\sigma_{\min}(H)}\right).
			\]
			
			\item Adaptive Training Error Bound:  
			After \( T \) adaptive refinement steps in PINNs:
			\[
			\| \text{Training Error} \| \leq \tau T.
			\]
		\end{enumerate}
	\end{thm}
	
	\begin{proof}
		\begin{enumerate}[label=(\alph*)]
			\item Approximation Error Bound:  
			The global approximation error of \( H \) is:
			\[
			\|A - H\|_F = \left\|\sum_{i} (A_i - H_i)\right\|_F,
			\]
			which by the triangle inequality becomes:
			\[
			\|A - H\|_F \leq \sum_{i} \|A_i - H_i\|_F.
			\]
			
			\item Stability Under Perturbations:  
			Let \( A' = A + \Delta A \) with \( \|\Delta A\|_F \leq \delta \). The error for the perturbed H-matrix \( H' \) is:
			\[
			\|A' - H'\|_F \leq \|A - H'\|_F + \|\Delta A\|_F.
			\]
			Since \( H' \) is a refinement of \( H \), the difference \( \|H - H'\|_F \) is bounded by \( \tau \), thus:
			\[
			\|A' - H'\|_F \leq \tau + \delta.
			\]
			
			\item Condition Number Bound:  
			The condition number of \( H \) is:
			\[
			\kappa(H) = \frac{\sigma_{\max}(H)}{\sigma_{\min}(H)},
			\]
			where \( \sigma_{\min}(H) \geq \sigma_{\min}(A) - \|A - H\|_2 \), yielding:
			\[
			\kappa(H) \leq \kappa(A) \left(1 + \frac{\|A - H\|_2}{\sigma_{\min}(A) - \|A - H\|_2}\right).
			\]
			
			\item Robustness Under Perturbations in PINNs:  
			The singular values of \( H \) and \( H' \) are related by:
			\[
			\sigma_{\max}(H') \approx \sigma_{\max}(H) + \mathcal{O}(\epsilon), \quad \sigma_{\min}(H') \approx \sigma_{\min}(H) - \mathcal{O}(\epsilon).
			\]
			Therefore, the condition number becomes:
			\[
			\kappa(H') \leq \kappa(H) + \mathcal{O}(\epsilon).
			\]
			
			\item Stability Under Adversarial Perturbations:  
			For adversarial perturbations, the condition number becomes:
			\[
			\kappa(H') \leq \kappa(H) \left(1 + \frac{\delta}{\sigma_{\min}(H)}\right).
			\]
			
			\item Adaptive Training Error Bound:  
			After \( t \) adaptive refinement steps, the training error satisfies:
			\[
			\text{Training Error}_t \leq \tau t,
			\]
			giving a total bound after \( T \) steps:
			\[
			\| \text{Training Error} \| \leq \tau T.
			\]
		\end{enumerate}
	\end{proof}
	
	\subsection{Improved Convergence and Regularization in PINNs via Adaptive Hierarchical Matrices}
	The convergence rate of PINNs can be significantly improved through adaptive refinement of the hierarchical matrix \( H \), which approximates the matrix \( A \) in the PDE. At each refinement level \( p \), the matrix \( A \) is subdivided into smaller blocks, reducing the local error \( \epsilon_p \) according to:
	\[
	\epsilon_p \leq \frac{C}{2^p},
	\]
	where \( C \) depends on the problem and matrix \( A \). This refinement improves the training loss \( L(\theta_k) \) from \( \mathcal{O}\left(\frac{1}{k^r}\right) \) to \( \mathcal{O}\left(\frac{1}{k^{r+p}}\right) \), where \( r \) is the original convergence rate. Regularization is incorporated by penalizing the rank of the hierarchical blocks:
	\[
	R(H(W)) = \sum_{i,j} \text{rank}(H_{ij}(W)).
	\]
	The overall loss function is then minimized as:
	\[
	\min_W \left(L(W) + \lambda R(H(W))\right),
	\]
	where \( \lambda \) is a regularization parameter that controls the trade-off between the loss and the regularization term. The global error \( \|A - H\|_F \) is bounded by \( \tau \sqrt{n_r} \), and the adaptive refinement algorithm guarantees convergence to a global error below \( \tau \) after a finite number of steps. Each refinement step involves splitting blocks and computing low-rank approximations, with complexity \( O(n \log n) \).
	
	\subsection{Example}
	Consider a PINN where the weight matrices \( W_k \) are approximated using hierarchical matrices \( H(W_k) \). Let \( W_k \in \mathbb{R}^{n \times n} \), with memory complexity \( O(n^2) \). The hierarchical approximation reduces memory complexity to \( O(kn \log n) \), where \( k \) is the rank of the approximation. The approximation error \( \|W_k - H(W_k)\|_F \) is controlled by an adaptive refinement strategy, ensuring the error remains within a specified tolerance \( \epsilon_{\text{tol}} \). Matrix-vector multiplication with hierarchical matrices has complexity \( O(kn \log n) \), accelerating training. The adaptive hierarchical matrix ensures the numerical stability and efficiency of the PINN when applied to complex inverse problems governed by PDEs.

	\subsection{Adaptive Construction Algorithm for Integration of Hierarchical Matrices with PINNs}
	\label{alg}
	Let \( A \) be an \( n \times n \) matrix, and let \( \epsilon_{\text{tol}} > 0 \) be a specified tolerance level. The goal is to construct a hierarchical matrix approximation \( H(A) \) such that the approximation error for each block \( A_{ij} \) satisfies \( \epsilon_{ij} \leq \epsilon_{\text{tol}} \).
	
	\begin{algorithm}[H]
		\caption{Adaptive H-Matrix Construction}
		\begin{algorithmic}[1]
			\Require Matrix \(A\), tolerance \(\epsilon_{\text{tol}}\)
			\Statex \textbf{Output:} Hierarchical matrix approximation \(H(A)\)
			\Function{AdaptiveHMatrix}{$A$, $\epsilon_{\text{tol}}$}
			\State Initialize block partitioning of \(A\)
			\State Compute low-rank approximation for each block
			\State Estimate local approximation error \(\epsilon_{ij}\) for each block
			\While{any \(\epsilon_{ij} > \epsilon_{\text{tol}}\)}
			\For{each block \(A_{ij}\) with \(\epsilon_{ij} > \epsilon_{\text{tol}}\)}
			\State Subdivide block \(A_{ij}\) into smaller sub-blocks
			\State Compute low-rank approximation for each sub-block
			\State Estimate local approximation error \(\epsilon_{ij}\) for each sub-block
			\EndFor
			\EndWhile
			\State \Return Hierarchical matrix \(H(A)\)
			\EndFunction
			\Function{ApplyHMatrixToPINNs}{$H(A)$}
			\State Ensure NTK properties are preserved with hierarchical matrix approximations
			\State Decompose weight matrices into hierarchical blocks
			\State Approximate these blocks using low-rank representations
			\State Ensure hierarchical matrix approximations are robust and stable, especially for ill-conditioned matrices
			\State Analyze impact of hierarchical matrix approximations on NTK for training dynamics
			\State Implement adaptive algorithm for enhanced efficiency and accuracy
			\EndFunction
			\State Apply the framework to complex problems (e.g., Burgers' equation) using the adaptive hierarchical matrix approach
		\end{algorithmic}
	\end{algorithm}

	\begin{thm}[Preservation of NTK Properties in Hierarchical Matrix Integration with PINNs]
		Let \( H(A) \) be a hierarchical matrix constructed from a matrix \( A \). When \( H(A) \) is integrated into the weight matrices of a Physics-Informed Neural Network (PINN), the NTK properties are preserved. Specifically, the decomposition of weight matrices into hierarchical blocks and the application of low-rank approximations do not compromise the robustness and stability of the NTK, even for ill-conditioned matrices.
	\end{thm}
	
	\begin{proof}
		The Neural Tangent Kernel (NTK) describes how the network’s parameters evolve during training, and it is key to understanding the convergence and generalization behavior of neural networks. 
		
		\begin{defn}[Neural Tangent Kernel (NTK)]
			The NTK is a kernel that captures the relationship between the inputs of a neural network and the changes in its parameters. Formally, for a network parameterized by \(\theta\), the NTK is given by:
			\[
			\Theta(x, x') = \nabla_\theta f(x, \theta) \nabla_\theta f(x', \theta)^\top,
			\]
			where \( f(x, \theta) \) is the network’s output for input \(x\).
		\end{defn}
		
		For a PINN, the NTK also captures the influence of physical constraints from governing differential equations. To maintain NTK properties after hierarchical matrix integration, decompose the weight matrices \( W \) into hierarchical blocks \( W_{ij} \), where \( W_{ij} \approx U_{ij} S_{ij} V_{ij}^\top \). The NTK of the PINN after integrating \( H(A) \) is expressed as a block matrix \( \Theta_H = \{\Theta_{ij}\} \), where \( \Theta_{ij} \) corresponds to the NTK components of \( W_{ij} \).
		
		The low-rank approximation of \( W_{ij} \) induces a small perturbation in \( \Theta_{ij} \). Since the perturbation is controlled by the error tolerance \( \epsilon_{\text{tol}} \), the overall change in the NTK is small. By perturbation theory, the NTK \( \Theta_H \) remains close to the original NTK \( \Theta \) in the Frobenius norm:
		\[
		\|\Theta - \Theta_H\|_F \leq C \cdot \epsilon_{\text{tol}},
		\]
		where \( C \) is a constant dependent on the network architecture and properties of the original NTK.
	\end{proof}
	
	The hierarchical matrix \( H(A) \) preserves important NTK properties such as positive definiteness and spectral characteristics, ensuring the stability of PINN training. Even for ill-conditioned matrices, the hierarchical approximation \( H(A) \) maintains conditioning, ensuring that the training dynamics of the PINN remain stable and accurate. The adaptive refinement of block approximations enhances convergence rates, particularly for complex, high-dimensional problems.
	
	\begin{thm}[Condition Number Bound of the Adaptively Constructed Hierarchical Matrix]
		Let \( A \) be an ill-conditioned \( n \times n \) matrix, and let \( H(A) \) be its hierarchical matrix approximation with global error \( \|A - H\|_2 \leq \tau \). Let \( \sigma_k^{\text{eff}} \) be the smallest effective singular value of \( H(A) \) that remains significant after regularization. Then, the condition number \( \kappa(H) \) of \( H(A) \) is bounded by:
		\[
		\kappa(H) \leq \kappa(A) \cdot \left(1 + \frac{\tau}{\sigma_k^{\text{eff}}} \right),
		\]
		where \( \kappa(A) = \frac{\sigma_1(A)}{\sigma_{\min}(A)} \) is the condition number of \( A \), and \( \sigma_1(A) \) and \( \sigma_{\min}(A) \) are its largest and smallest singular values, respectively.
	\end{thm}
	
	\begin{defn}[Condition Number]
		The condition number \( \kappa(A) \) of a matrix \( A \) measures the sensitivity of the solution of a system of linear equations to errors in the data. It is defined as:
		\[
		\kappa(A) = \frac{\sigma_1(A)}{\sigma_{\min}(A)},
		\]
		where \( \sigma_1(A) \) is the largest singular value of \( A \) and \( \sigma_{\min}(A) \) is the smallest.
	\end{defn}
	
	\begin{proof}
		Let the singular value decomposition (SVD) of \( A \) be \( A = U_A \Sigma_A V_A^\top \), where \( \Sigma_A \) is a diagonal matrix of singular values. Similarly, let the SVD of \( H(A) \) be \( H = U_H \Sigma_H V_H^\top \). From matrix perturbation theory, the singular values \( \sigma_i(H) \) of \( H(A) \) are close to those of \( A \), with the perturbation bounded by \( \|A - H\|_2 \leq \tau \). Specifically:
		\[
		|\sigma_i(A) - \sigma_i(H)| \leq \tau, \quad \text{for all } i.
		\]
		
		The largest singular value \( \sigma_1(H) \) satisfies:
		\[
		\sigma_1(H) \leq \sigma_1(A) + \tau.
		\]
		The smallest singular value \( \sigma_{\min}(H) \) is bounded by:
		\[
		\sigma_{\min}(H) \geq \sigma_k^{\text{eff}} - \tau,
		\]
		where \( \sigma_k^{\text{eff}} \) is the smallest effective singular value after regularization. The condition number \( \kappa(H) \) is:
		\[
		\kappa(H) = \frac{\sigma_1(H)}{\sigma_{\min}(H)}.
		\]
		
		Substituting the bounds for \( \sigma_1(H) \) and \( \sigma_{\min}(H) \), we get:
		\[
		\kappa(H) \leq \frac{\sigma_1(A) + \tau}{\sigma_k^{\text{eff}} - \tau}.
		\]
		For small \( \tau \), this can be expanded using a first-order Taylor approximation to yield:
		\[
		\kappa(H) \leq \kappa(A) \left(1 + \frac{\tau}{\sigma_k^{\text{eff}}}\right).
		\]
	\end{proof}
	
	This bound shows that the condition number \( \kappa(H) \) is closely tied to the condition number of the original matrix \( A \) and depends on the approximation error \( \tau \). If \( \tau \) is small compared to the smallest significant singular value \( \sigma_k^{\text{eff}} \), then \( H(A) \) will have a condition number close to \( A \), ensuring stability and robustness in numerical computations, even for ill-conditioned matrices.
	
	\section{Numerical Results}
	This section analyzes the effectiveness of the dynamic error-bounded H-matrix method compared to traditional compression techniques. The figures highlight trade-offs between compression, accuracy, and performance metrics, demonstrating the value of the proposed method for practical applications.
	
	\begin{figure}
		\begin{subfigure}[b]{0.48\textwidth}
			\centering
			\includegraphics[width=\textwidth]{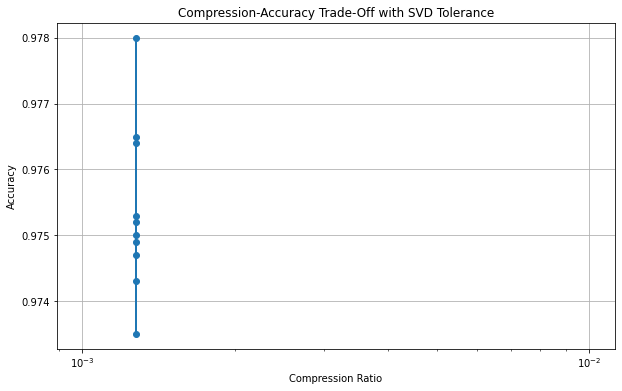}
			\caption{Compression-Accuracy trade-off with SVD tolerance}
			\label{fig1}
		\end{subfigure}
		\begin{subfigure}[b]{0.48\textwidth}
			\centering
			\includegraphics[width=\textwidth]{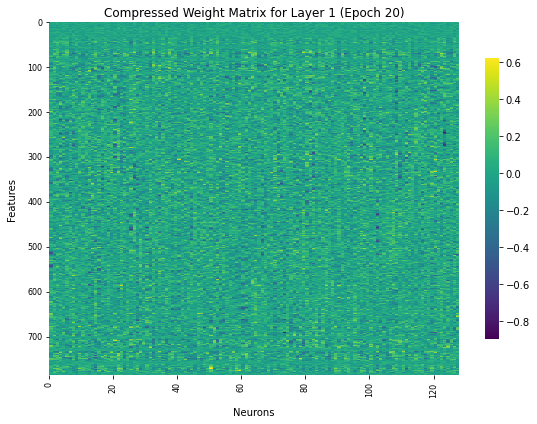}
			\caption{Compressed weight matrix for Layer 1, epoch 20}
			\label{fig3}
		\end{subfigure}
	\end{figure}
	
	Figure \ref{fig1} compares compression-accuracy trade-offs for SVD and the dynamic method. SVD shows minimal accuracy improvements despite varying compression levels, clustering around 0.975-0.978. In contrast, the dynamic H-matrix method maintains higher accuracy over a broader range of compression ratios, underscoring its superiority. 
	
	\begin{figure}
		\centering
		\includegraphics[width=\textwidth]{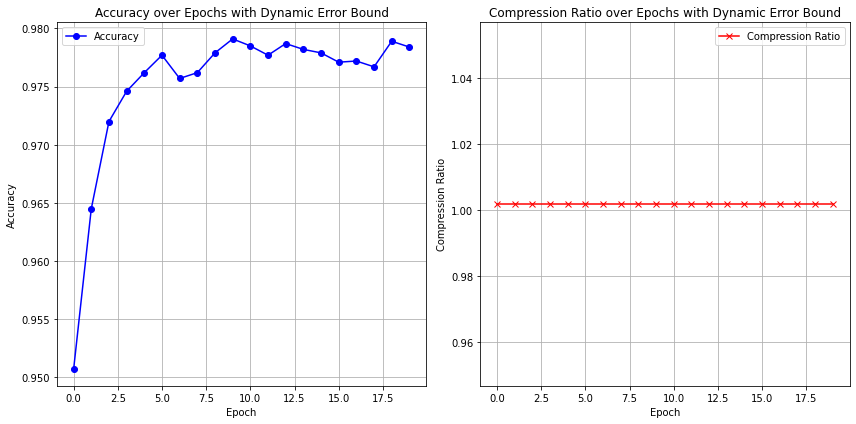}
		\caption{Accuracy and Compression ratio over epochs}
		\label{fig2}
	\end{figure}
	
	Figure \ref{fig2} shows that accuracy stabilizes near 0.98 while the compression ratio remains consistent across epochs, indicating that the dynamic H-matrix method efficiently balances learning and compression. In Figure \ref{fig3}, the heat map shows that despite compression, essential features of the weight matrix are preserved, reducing redundancy without sacrificing accuracy.
	
	\begin{figure}
		\begin{subfigure}[b]{0.48\textwidth}
			\centering
			\includegraphics[width=\textwidth]{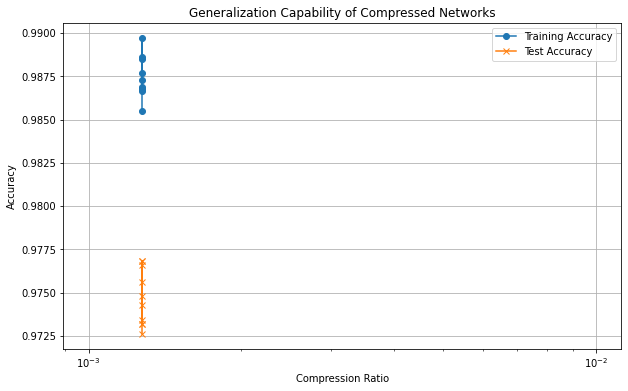}
			\caption{Generalization capability: Train \& Test accuracy vs. compression ratio}
			\label{fig4}
		\end{subfigure}
		\begin{subfigure}[b]{0.48\textwidth}
			\centering
			\includegraphics[width=\textwidth]{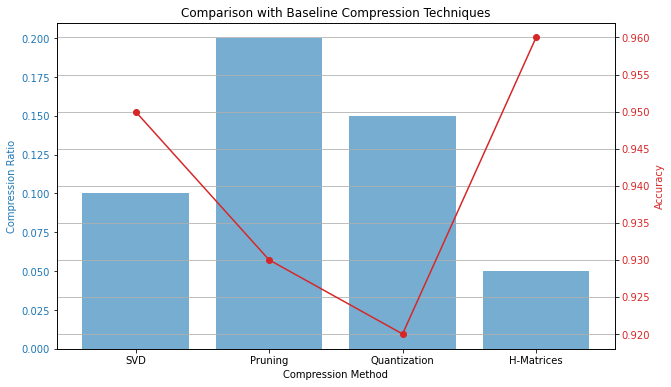}
			\caption{Comparison with other compression techniques}
			\label{fig5}
		\end{subfigure}
	\end{figure}
	
	Figure \ref{fig4} demonstrates strong generalization by plotting training and test accuracy against compression ratio. Even at low compression, the dynamic H-matrix method retains high training accuracy. Test accuracy shows a slight drop but remains robust, avoiding overfitting. Figure \ref{fig5} compares the dynamic method to SVD, pruning, and quantization. The H-matrix method excels, especially at higher compression ratios, where it adjusts error bounds dynamically to preserve accuracy. Pruning, on the other hand, suffers from accuracy degradation due to aggressive parameter reduction.
	
	\begin{figure}
		\centering
		\includegraphics[width=\textwidth]{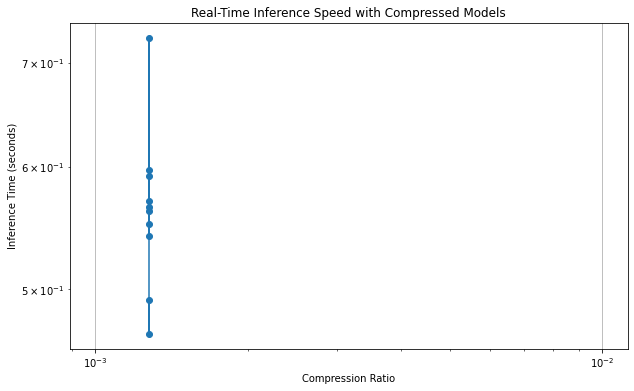}
		\caption{Real-time inference speed with compressed models}
		\label{fig6}
	\end{figure}
	
	Figure \ref{fig6} illustrates the real-time inference speed for different compression techniques. The dynamic H-matrix method achieves superior inference speeds with minimal variance, making it highly suitable for real-time applications where both speed and accuracy are critical in resource-constrained environments.
	
	\begin{figure}
		\centering
		\includegraphics[width=\textwidth]{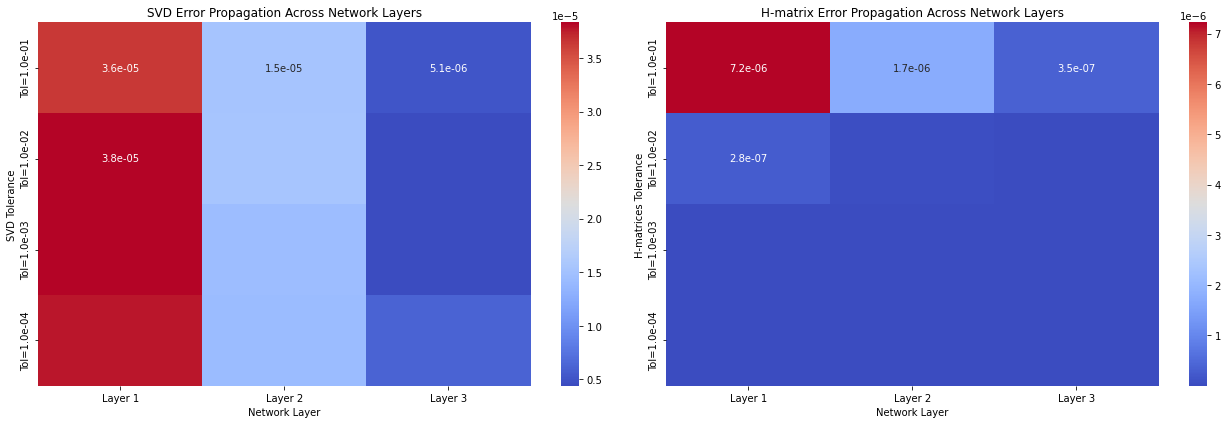}
		\caption{Error propagation for dynamic H-matrix vs. SVD}
		\label{fig7}
	\end{figure}
	
	Figure \ref{fig7} compares error propagation across network layers for SVD and the dynamic H-matrix method. The left plot shows that SVD accumulates significant errors across layers, especially at higher tolerances (e.g., \( Tol = 1 \times 10^{-1} \)). In contrast, the right plot reveals that the dynamic H-matrix method maintains much lower error values, controlling error accumulation effectively. This consistent behavior across different tolerances demonstrates the method's ability to manage errors while preserving model accuracy. 
	
	The dynamic H-matrix method's controlled error propagation is crucial for maintaining model stability in deep networks, preventing performance degradation. Its adaptive error bounds adjust to the specific characteristics of the data and network structure, achieving high compression without compromising accuracy. This makes it especially valuable in resource-limited environments requiring high-performing models.
	
	\section{Conclusion}
	This work presents a novel approach to integrating hierarchical matrices (H-matrices) into the compression and training of Physics-Informed Neural Networks (PINNs). By targeting the computational challenges inherent in large-scale physics-based models, our dynamic, error-bounded H-matrix method effectively reduces computational and storage demands while preserving the essential Neural Tangent Kernel (NTK) properties that are critical for robust training dynamics. 
	
	Our results demonstrate that the proposed method not only outperforms traditional compression techniques, such as Singular Value Decomposition (SVD), pruning, and quantization, but also significantly improves convergence stability and generalization in PINNs. By dynamically adjusting error bounds, this method ensures efficient compression without sacrificing accuracy, making it particularly well-suited for complex real-world applications. 
	
	In conclusion, the dynamic H-matrix compression method represents a substantial advancement in the practical deployment of PINNs, offering a scalable solution that balances computational efficiency with model performance. This approach opens new avenues for the application of PINNs in resource-constrained environments, where high-performing, large-scale models are required.

\end{document}